\documentclass[format=sigconf, balance=false,nonacm]{acmart}
\AtBeginDocument{%
  \providecommand\BibTeX{{%
    \normalfont B\kern-0.5em{\scshape i\kern-0.25em b}\kern-0.8em\TeX}}}

\usepackage[american]{babel}
\usepackage{microtype}
\usepackage{subcaption}
\usepackage{paralist}
\usepackage{cleveref}
\usepackage{amsmath,amssymb, amsthm}

\DeclareMathOperator{\prom}{prom}
\DeclareMathOperator{\iso}{iso}
\DeclareMathOperator{\mindesc}{mindesc}

\begin{document}
\let\cref\Cref

\title{Orometric Methods in Bounded Metric Data}

\author{Maximilian Stubbemann}
\affiliation{%
\institution{L3S Research Center and University of Kassel}
\city{Kassel}
\country{Germany}
}
\email{stubbemann@l3s.de}

\author{Tom Hanika}
\affiliation{%
\institution{Berlin School of Library and Information Science\\ Humboldt University of Berlin}
\city{Berlin}
\country{Germany}}
\email{tom.hanika@hu-berlin.de}

\author{Gerd Stumme}
\affiliation{%
\institution{University of Kassel and L3S Research Center}
\city{Kassel}
\country{Germany}
}
\email{stumme@cs.uni-kassel.de}

\setcopyright{acmcopyright}
\copyrightyear{2019}
\acmYear{2019}
\acmDOI{TBA}
\acmConference[--]{---}{---}{---}

\begin{abstract}
  A large amount of data accommodated in knowledge graphs (KG) is actually
  metric. For example, the Wikidata KG contains a plenitude of metric facts
  about geographic entities like cities, chemical compounds or celestial
  objects. In this paper, we propose a novel approach that transfers orometric
  (topographic) measures to bounded metric spaces. While these methods were
  originally designed to identify relevant mountain peaks on the surface of the
  earth, we demonstrate a notion to use them for metric data sets in
  general. Notably, metric sets of items inclosed in knowledge graphs. Based on
  this we present a method for identifying outstanding items using the
  transferred valuations functions ’isolation’ and ’prominence’. Building up on
  this we imagine an item recommendation process. To demonstrate the relevance
  of the novel valuations for such processes we use item sets from the Wikidata
  knowledge graph. We then evaluate the usefulness of ’isolation’ and
  ’prominence’ empirically in a supervised machine learning setting. In
  particular, we find structurally relevant items in the geographic population
  distributions of Germany and France.
\end{abstract}



\maketitle

\section{Introduction}

Knowledge graphs, such as DBpedia~\cite{Lehmann15} or
Wikidata~\cite{Kroetzsch14}, are the state of the art structure for
storing information and to draw knowledge from. They are knowledge
bases represented as graphs and consist essentially of \emph{items}
which are related through \emph{properties} and \emph{values}. This
enables them to fulfill the task of giving exact answers to exact
questions. However, they are limited when it comes to provide a
concise overview of the contained metric data and give characteristic
insights.  For example, the number of such metric data sets in
Wikidata is tremendous: since (presumably) the set of all cities of
the world, including their geographic coordinates, is included in
Wikidata, this constitutes a metric data set. Further examples are
chemical compounds and their physical properties like mass and size or
celestial bodies and their trajectories.

One possibility to enhance the understanding of the metric data is to
identify outstanding elements, i.e., outstanding items. Based on such
elements it is possible to compose or enhance item recommendations to
users. For example, such recommendations could provide a set of the
most relevant cities in the world with respect to being outstanding in
their local surroundings. However, it is a challenging task to
identify outstanding items in metric data sets. In cases where the
metric space is equipped with an additional valuation function, this
task becomes more feasible. Such functions, often called \emph{scores}
or \emph{height} function, are frequently naturally provided: cities
may be ranked by their population; the importance of scientific
publications maybe ranked by the $h$-index \cite{Hirsch05} of their
corresponding authors. A naïve approach for recommending relevant
items in such settings would be based on the claim: items with higher
scores are more relevant items. As this method seems reasonable for
many applications, some obstacles may arise if the ``highest'' items
of the topic may be concentrated into a specific region of the
underlying metric data space. For example, returning the twenty most
populated cities in the world as an overview for the city landscape
would return no European city\footnote{\url{https://en.wikipedia.org/wiki/List_of_largest_cities} on
  2019-06-16}, recommending the hundred highest mountain peaks of the
world would not lead to any knowledge about the mountainscapes
outsides of Asia\footnote{\url{https://en.wikipedia.org/wiki/List_of_highest_mountains_on_Earth}
  on 2019-06-16}.

To overcome this problem we propose a novel approach: we combine the
valuation measure (e.g., ``height'') and distances drawn from the
metric in order to provide new valuation functions on the set of
items, called \emph{prominence} and \emph{isolation}. In contrast to
the naïve approach, those functions do value an item based on its
height in relation to the valuations of the surrounding items. This
results in a valuation function on the set of items that reflects the
extend to which an item is locally outstanding. The basic idea behind
the novel valuation functions is the following. The prominence
function values an item based on the minimal descent (with respect to
the height function) that is needed to get to another point of at
least same height. Furthermore, the isolation function, sometimes also
called \emph{dominance radius}, values the distance to the next higher
point with respect to the given metric and height function. These
measures are adapted from the field of topography where topographic
isolation and topographic prominence are used in order to identify
outstanding mountain peaks. Our approach is based on~\cite{Schmidt18},
where the authors Schmidt \& Stumme proposed prominence and dominance
for networks. We will transfer and adapt these through generalization
to the realm of bounded metric space.

To give a first insight to the potential of the novel valuation
functions in knowledge graphs, we will empirically verify their
ability to identify relevant items for a given topic. For this we
employ a supervised machine learning task. We evaluate if isolation
and prominence functions can contribute to the task of identifying
relevant items in the sets of French and German cities.

The contributions of this paper are as follows:
\begin{inparaitem}
\item We propose prominence and isolation for bounded metric
  spaces. For this we generalize the results in~\cite{Schmidt18} which
  were limited to finite, undirected graphs.
\item We demonstrate an artificial machine learning task for
  evaluating novel valuation functions in metric data.
\item We introduce a general approach for using prominence and
  isolation to enrich metric data in knowledge graphs. We show
  empirically that this information helps to identify a set of
  representative items.
\end{inparaitem}

The remainder of this paper is organized as follows. In~\cref{sec:rel}
we give a short overview over related work. This is followed
by~\cref{sec:math} were the necessary mathematical foundation is laid
out. \cref{sec:appl} gives a first insight in how the novel valuation
functions can be employed in a possible recommendation process. We
evaluate this in~\cref{sec:exp} and conclude our work within~\cref{sec:conc}.

\section{Related Work}
\label{sec:rel}

Item recommendations for knowledge graphs is a contemporary topic of high
interest in research. Investigations cover for example music recommendation
using content and collaborative information~\cite{Oramas16} or movie
recommendations using PageRank like methods~\cite{Catherine16}. The former is
based on the common notion of embedding, i.e., embedding of the graph structure
into $d$-dimensional $\mathbb{R}$ vector spaces. The latter operates on the
relational structure itself. Our approach differs from those as it is based on
combining a valuation measure with the metric of the data space. Nonetheless,
given an embedding into an finite dimensional $\mathbb{R}$ vector space, one
could apply isolation and prominence in those as well. 

The novel valuation functions prominence and isolation are inspired by
topographic measures, which have their origin in the classification of mountain
peaks. The idea of ranking peaks solely by their absolute height was already
deprecated in 1978 by Fry in his work~\cite{Fry87}. The author introduced
prominence for geographic mountains, a function still investigated in this
realm, e.g., in Torres et. Al.~\cite{Torres18}, where the authors use deep
learning  methods to identify prominent mountain peaks. Another
recent step for this was made in~\cite{Kirmse17}, where the authors investigated
methods for discovering new ultra-prominent mountains. Isolation and more
valuations functions motivated in the orometric realm are collected
in~\cite{Helman05}.

Recently the idea of transferring orometric functions to different realms of
research gained attention: The authors of \cite{Nelson19} used topographic
prominence to identify population areas in several U.S.\ States.
In~\cite{Schmidt18} the authors Schmidt \& Stumme transferred prominence and
dominance, i.e., isolation, to co-author graphs in order to evaluate their
potential of identifying ACM Fellows. We build on this for proposing our
valuation functions on bounded metric data. This generalization results in a
wide range of applications.


\section{Mathematical Modeling}
\label{sec:math}
Let us consider the following scenario: We have a data set $M$, consisting of a
set of items, in the following called \emph{points}, equipped with a metric
 $d$ and a valuation function $h$, in the following called \emph{height
  function}. The goal of the orometric (topographic) measures prominence and
isolation is, to provide measures that reflect the extend to which a point is
locally outstanding in its neighborhood. 

Let $M$ be a non-empty set and $d: M \times M \to \mathbb{R}_{\geq 0}$. We call
$d$ a \emph{metric} on the set $M$ iff
\begin{inparaenum}
\item $\forall x,y \in M:d(x,y)=0\iff x=y$, and 
\item $d(x,y)=d(y,x)$ for all $x,y \in M$, called symmetry,  and
\item $\forall x,y,z \in M: d(x,z) \leq d(x,y) + d(y,z)$, called triangle inequality.
\end{inparaenum}
If $d$ is a metric on $M$, we call $(M,d)$ a \emph{metric space} and if $M$ is
finite we call $(M,d)$ a \emph{finite metric space}. If there exists a
$C\in\mathbb{R}_{\geq0}$ such that we have $d(m,n)\leq C$ for all $m,n \in M$ ,
we call $(M,d)$ \emph{bounded}. For the rest of our work we assume that $|M|>1$
and $(M,d)$ is a bounded metric space. Additionally, we have that $M$ is equipped with a height
function (valuation / score function) $h:M \to \mathbb{R}_{\geq 0}, m\mapsto h(m)$.

\begin{definition}[Isolation]\label{def:isolation}
  Let $(M,d)$ be a bounded metric space and let $h: M \to \mathbb{R}_{\geq 0}$ be
  a height function on M. The \emph{isolation of a point} $x \in M$ is then
  defined as follows:
  \begin{itemize}
  \item If there is no point with at least equal height to $m$,
    than $\iso(m)\coloneqq\sup\{d(m,n)\mid n \in M\}$. The boundedness of $M$ guarantees
    the existence of this suprenum. 
  \item If there is at least one other point in $M$
    with at least equal height to $m$, we define its isolation by:
    \begin{equation*}
      \iso(m)\coloneqq\inf\{d(m,n)\mid n \in M \setminus \lbrace m \rbrace \wedge h(n) \geq h(m)\}.
    \end{equation*}
  \end{itemize}
\end{definition}

The isolation of a mountain peek is often called the \emph{dominance radius} or
sometimes the \emph{dominance}. Since the term \emph{orometric dominance} of a
mountain sometimes refers to the quotient of prominence and height, we
will stick to the term \emph{isolation} to avoid confusion.

While the isolation can be defined within the given setup, we have to equip our
metric space with some more structure in order to transfer the notion of
prominence. Informally, the prominence of a point is given by the minimal
vertical distance one has to descend to get to a point of at least the same
height. To adapt this measure to our given setup in metric spaces with a height
function, we have to define what a path is. Structures that provide paths in a
natural way are graph structures. For a given graph $G=(V,E)$ with vertex set
$V$ and edge set $E\subseteq{V\choose{2}}$, \emph{walks} are defined as
sequences of nodes $\{v_i\}_{i=0}^n$ which satisfy $\{v_{i-1},v_i\} \in E$ for
all $i \in \lbrace 1,...,n \rbrace$. If we also have $v_i \neq v_j$ for
$i \neq j$, we call such a sequence a \emph{path}. For $v,w\in V$ we say $v$ and
$w$ are \emph{connected} iff there exists path connecting them. Furthermore, we
denote by $G(v)$ the \emph{connected component} of $G$ containing $v$, i.e.,
$G(v)\coloneqq\{w\in V\mid v\ \text{is connected with}\ w\}$.

To use the prominence measure as introduced by Schmidt
\& Stumme in~\cite{Schmidt18}, which is indeed defined on graphs, we have to
derive an appropriate graph structure from our metric space.
 
The topic of graphs embedded in finite dimensional vector spaces, so called
spatial networks \cite{Barthelmy11}, is a topic of current interest. These
networks appear in real world scenarios frequently, for example in the modeling
of urban street networks \cite{Jiang04}.  Note that our setting, in contrast to
the afore mentioned, is not based on a priori given graph structure. In our
scenario the graph structure must be derived from the structure of the given
metric space.

Our approach is, to construct a \emph{step size graph} or
\emph{threshold graph}, where we consider points in the metric space as nodes and
connect two points through an edge, iff their distance is smaller then a given
threshold $\delta$. 

\begin{definition}($\delta$-Step Graph)\label{def:stepgraph}
Let $(M,d)$ be a metric space and $\delta > 0$. We define the
\emph{$\delta$-step graph} or \emph{$\delta$-threshold graph}, denoted by
$G_{\delta}$, as the tuple $\left( M,E_\delta \right)$ via
\begin{equation}
  E_{\delta}\coloneq \{ \{ m, n\} \in {M\choose{2}} \mid d(m,n) \leq \delta \}\}
\end{equation}
\end{definition}

This approach is similar to the one found in the realm of random
geometric graphs, where it is common sense to define random graphs by
placing points uniformly in the plane and connect them via edges if
their distance is less than a given threshold~\cite{Penrose2003}.

Since we introduced a possibility to derive a graph that just depends
on the metric space, we use a slight modification of the definition of
prominence compared to~\cite{Schmidt18} for networks.

\begin{definition}[Prominence in Networks]\label{def:prom}
  Let $G=(V,E)$ be a graph and let $h:V \to \mathbb{R}_{\geq 0}$ be a
  height function. The \emph{prominence} $\prom_G(v)$ of $v\in V$ is
  defined by
  \begin{equation}
    \label{eq:prom}
    \prom_{G}(v)\coloneqq\min\{h(v),\mindesc_{G}(v)\}
  \end{equation}
  where
  $\mindesc_{G}(v)\coloneqq \inf\{\max\{h(v)-h(u)\mid u\in p\}\mid p
  \in P_v\}$. The set $P_{v}$ contains all paths to vertices $w$ with
  $h(w)\geq h(v)$, i.e.,
  $P_{v}\coloneqq\{\{v_i\}_{i=0}^n \in P\mid v_0=v\wedge v_n\neq v  \wedge
  h(v_n)\geq h(v)\}$,
  where $P$ denotes the set of all paths of the graph $G$.
\end{definition}

Informally, $\mindesc_{G}(v)$ reflects on the minimal descent in order
to get to a vertex in $G$ which has a height of at least $h(v)$. For
this the definition makes use of the fact that $\inf\emptyset=\infty$
in cases where no such point exists. This case results in $\prom_G(v)$
being the height of $v$.  An essential distinction to the prior
definition in~\cite{Schmidt18} is, that we now consider all paths and
not just shortest paths. Based on this we are able to transfer the
notions above to metric spaces.

\begin{definition}[$\delta$-Prominence in Metric Spaces]\label{def:prominences}
  Let $(M,d)$ be a bounded metric space and
  $h:M \to \mathbb{R}_{ \geq 0 }$ be a height function. We define the
  $\delta$-prominence $\prom_\delta (m)$ of $m \in M$ as
  $\prom_{G_{\delta}}(v)$, i.e, the prominence of $m$ in the step
  graph $G_{\delta}$ from~\cref{def:stepgraph}.
\end{definition}

We now have a prominence term for all metric spaces that depends on a
parameter $\delta$ to choose. For all knowledge procedures, choosing
such a parameter is a demanding task.  Hence, we want to provide in
the following a natural choice for $\delta$. The ideas for this is
informally the following: We consider only those values for $\delta$
such that corresponding $G_{\delta}$ does not exhibit noise, i.e.,
there is no element without a neighbor. In other words, we allow only
those values of $\delta$ such that
$\forall m\in M\exists e\in E_\delta:m\in e$.

\begin{definition}[Minimal Threshold]\label{def:minimal-delta}
  For $(M,d)$ a bounded metric space with $|M|>1$ we define the
  \emph{minimal threshold} $\delta_M$ of $M$ as
  \begin{equation*}
    \delta_{M}\coloneqq \sup\{\inf \{d(m,n)\mid n \in M \setminus\{m\}\}\mid m \in M\}.
  \end{equation*}
\end{definition}

Based on this definition a natural notion of prominence for metric
spaces (equipped with a height function) emerges via a limit process.

\begin{lemma}
  Let $M$ be a bounded metric and $\delta_M$ as
  in~\cref{def:minimal-delta}. For $m \in M$ the descending limit
  \begin{equation}
    \lim_{\delta \searrow \delta_M} \prom_\delta (m)
  \end{equation}
  exists.
\end{lemma}

\begin{proof}
  Fix any $\hat{\delta} > \delta_M$ and consider on the open interval
  from $\delta_M$ to $\hat{\delta}$ the function that maps $\delta$ to
  $\prom_\delta(m)$:
  \[\prom_{(.)}(m) :\ ]\delta_M, \hat{\delta}[ \to \mathbb{R}, \delta
    \mapsto \prom_\delta (m).\] It is well known that it is sufficient
  to show that $\prom_{(.)}(m)$ is monotone decreasing and bounded
  from above. Since we have for any $\delta$ that
  $\prom_\delta (m) \leq h(m)$ holds, we need to show the monotony.
  Let $\delta_1, \delta_2$ be in $] \delta_M, \hat\delta [$ with
  $\delta_1 \leq \delta_2$. If we consider the corresponding graphs
  $(M,E_{\delta_1})$ and $(M, E_{\delta_2})$, it easy to see
  $E_{\delta_1} \subseteq E_{\delta_2}$. Hence, we have to consider
  more paths in~\cref{eq:prom} for $E_{\delta_2}$, resulting in a not
  larger value for the infimum. We obtain
  $\prom_{\delta_1}(m) \geq \prom_{\delta_2}(m)$, as required.
\end{proof}

This leads in a natural way directly to the following definition. 

\begin{definition}[Prominence in Metric Spaces]\label{def:metric-prom}
  If $M$ is a bounded metric space with $|M|>1$ and a height function $h$, the
  prominence $\prom(m)$ of $m$ is defined as:
\begin{equation}
\prom(m) \coloneqq \lim_{\delta \searrow \delta_M} \prom_\delta (m).
\end{equation}
\end{definition}

Note, if we want to compute prominence on a real world finite metric data set,
it is possible to directly compute the prominence values: in that case the
supremum in \cref{def:minimal-delta} can be replaced by a maximum and the
infimum by a minimum, which leads to $\prom(m)$ being equal to
$\prom_{\delta_M}(m)$.  Hence, we can compute prominence and isolation for every
point in the finite data set. There are results for efficiently creating such
threshold graphs~\cite{Bentley75}. However, for our needs in this work, in
particular in the experiment section, a quadratic brute force approach for
generating all edges is sufficient.
We want to show that our prominence definition for bounded metric spaces is a
natural generalization of~\cref{def:prom}.

\begin{lemma}
  Let $G=(V,E)$ be a finite, connected graph with $ \left| V \right| \geq
  2$. Consider $V$ equipped with the shortest path metric as a  metric space. Then
  the prominence $\prom_{G}(\cdot)$ from~\Cref{def:prom} and $\prom(\cdot)$
  from~\cref{def:metric-prom} coincide.
\end{lemma}

\begin{proof}
  Let $M\coloneqq V$ be equipped with the shortest path metric $d$ on $G$. As $G$
  is connected and has more than one node, we have $\delta _M=1$. This yields
  that $(M,E_{\delta_M})$ from~\Cref{def:stepgraph} and $G$ are equal. Hence,
  the prominence terms coincide.
\end{proof}

\section{Application}
\label{sec:appl}

\subsection{Score based item recommending}

As an application of our valuation functions, we envisage a general approach for
a score based item recommending process. The task of item recommending in
knowledge graphs is a current research topic. However, most approaches are
solely based on knowledge about preferences of the given user and graph
structural properties, often accessed through knowledge graph embeddings. The
idea of the recommendation process we imagine differs from those. We stipulate
on a procedure that is based on the information entailed in the connection of
the metric aspects of the data together with some (often present) height
function. Of course, we are aware that this limits our approach to metric data
in knowledge graphs, only. Nonetheless, given the large amounts of metric item
sets in prominent knowledge graphs, we claim the existence of a plenitude of
applications. For example, while considering sets of cities, such a system could
recommend a \emph{relevant} subset, based on a height function, like population,
and a metric, like geographical distances. By doing so, we introduce a source of
information for recommending metric data in relational structures, like
knowledge graphs. A common approach for analyzing and learning in knowledge
graphs is knowledge graph embedding. There is an extensive amount of research
about that, see for example~\cite{Wang14,Bordes11}. Since our novel methods rely
solely on bounded metric spaces and some valuation function, one may apply those
after the embedding step as well. In particular, one may use isolation and
prominence for investigating or completing knowledge graph embeddings. This
constitutes our second envisioned application. Finally, common item recommending
scores/ranks can also be used as height functions in our sense. Hence, computing
prominence and isolation for already setup recommendation systems is another
possibility. Here, our valuation functions have the potential to enrich the
recommendation process with additional information. In such a way our measures
can provide a novel additional aspect to existing approaches.

The realization and evaluation of our proposed recommendation approach is out of
scope of this paper. Nonetheless, we want to provide some first insights for the
applicability of valuation functions for item sets based on empirical
experiments. As a first experiment, we will evaluate if isolation and prominence
help to separate important and unimportant items in specific item sets in
Wikidata. More specifically, we will evaluate if the valuation functions help to
differentiate important and unimportant municipalities in the countries of
France and Germany, solely based on their geographic metric properties and their
population as height function.

\subsection{Enriching metric item sets in Wikidata}
In this section we depict an universal approach for enriching finite metric item
sets in Wikidata using the introduced functions isolation and prominence. In
order to enhance the grasp for the reader, we accompany every step with a
running example. To which extend do municipalities stand out with respect to
their local surroundings, based on population (height)? The particular
steps are as follows:

\begin{enumerate}
\item \textbf{Identify a metric item set in the knowledge graph:} For this we
  need to identify the metric space of all items in some considered set. One may
  pre-compute their pairwise distances, if applicable.

  For our experiments we identify the set of German municipalities and French
  municipalities with their geographic coordinates in longitude and latitude and
  compute as well their pairwise (approximated) distances.
\item \textbf{Identify height function:} Since we want to compute the prominence
  and isolation of the items, we also have to identify a height function. Hence,
  we need to identify a valued property shared by all items identified in the
  step above which is also relevant to the enriching task.

  In our running example we identify the population of the municipalities as
  such a relevant shared valued property. 

\item \textbf{Compute isolation:} Based on the steps before we are now
  abled to compute the isolation for all items in the item sets.

  For our running example, we compute the isolation for all municipalities for
  the item sets of Germany and France. 

\item \textbf{Compute the threshold graph:} For computing the prominence values
  for all items in the item sets,  we need to compute the threshold  graph and the threshold $\delta_{M}$
  using~\cref{def:stepgraph} and ~\cref{def:minimal-delta}.

  In our running examples, for Germany we compute the value $\delta_{M}\approx 32$
  kilometers. This value is necessary in order to preserve a connection between
  Borkum (Q25082) and Krummhörn (Q559432). For the French item set we compute
  $\delta_{M}\approx 54$ kilometers in order to preserve the connection between
  Mende (Q191772) and La Grand-Combe (Q239967).

\item \textbf{Compute prominences:} Equipped with the threshold graph we are now
  able to compute the prominence values for all items
  using~\cref{def:prominences}.
\end{enumerate}

\subsection{Resulting Questions}
The sections above raise the natural question for an objective evaluation of the
functions prominence and isolation. In this section we present such an
evaluation scheme by means of two qualitative questions connected to this task.

Assume we have given a bounded metric space $M$ representing our data set and a
given height function $h$.  The aim of the research questions we propose in the
following is to evaluate if our functions isolation and prominence provide
useful information about the relevance of given points in the metric space. If
$(M,d,h)$ is a metric space equipped with an additional height function, let the
map $c:M \to \lbrace 0,1 \rbrace$ be a binary function that classifies the
points in the data set as relevant (1) or not (0). We want to answer the
following question to evaluate if there is a connection between the extent to
which a data point is local outstanding (i.e., has high isolation and
prominence) and relevance.  We connect this to our running example using the
classification function that classifies municipalities having a university (1)
and municipalities that do not have an university (0). We admit that the
underlying classification is not meaningful in itself. However, since this setup
is essentially a benchmark framework (in which we assume cities with
universities to be more relevant) we refrain from employing a more meaningful
classification task in favor of a controllable classification scenario.

\begin{enumerate}
\item \textbf{Are prominence and isolation alone characteristical for
    relevance?} \newline We use isolation and/or prominence for a given set of
  data points as features. To which extend do these features improve learning a
  classification function for relevance?

  This question manifests in our running example as follows: are prominence and
  isolation useful features to classify the university locations of France and
  Germany?
\item \textbf{Do prominence and isolation provide additional information, not
    catered by the absolute height?} \newline Do prominence and isolation
  improve the prediction performance of relevance compared to just
  using the absolute height?  Does a classifier that uses prominence and
  isolation as additional features produce better results than a classifier
  that just uses the absolute height?

  In the context of our running example: Do prominence and isolation of
  municipalities add information to the population feature, that help to
  characterize the university locations, compared to using the plain population
  value?
\end{enumerate}

We will evaluate the proposed setup in the realm a knowledge graph and take on
the questions stated above in the following section and present some
experimental evidence.

\section{Experiments}
\label{sec:exp}

\subsection{Dataset}
We extract information about municipalities in the countries of Germany and
France from the Wikidata knowledge graph. This knowledge graph is a structure
that stores knowledge via \emph{statements}, linking \emph{entities} via
\emph{properties} to \emph{values}. A detailed description can be found in
\cite{Kroetzsch14}, while \cite{Hanika19} gives an explicit mathematical
structure to the Wikidata graph and shows how to use the graph for extracting
implicational knowledge from Wikidata subsets. We investigate in the following
if prominence and isolation of a given municipality can be used as features to
predict university locations in a classification setup.  We use the query
service of Wikidata\footnote{\url{https://query.wikidata.org/}} to extract
points in the country maps from Germany and France and to extract all their
universities. For every relevant municipality we extract the coordinates and the
population.  The necessary SPAQRL queries we employed for all the followings
tasks are documented in our GitHub
repository\footnote{\url{https://github.com/mstubbemann/Orometric-Methods-in-Bounded-Metric-Data}}
for our paper project. While constructing the needed metric space, we have to
overcome some obstacles.

\begin{itemize}
\item Wikidata provides different relations for extracting items that are
  instances of the notion city. The most obvious choice is to employ the
  \emph{instance of} (P31) property for the item \emph{city} (Q515). Using this,
  including \emph{subclass of} (P279), we find insufficient results for
  generating our data sets. More specific, we find only 102 French cities and
  2215 German cities.\footnote{Queried on 07-08-19} For Germany, there exists a
  more commonly used item \emph{urban municipality of Germany} (Q42744322) for
  extracting all cities, while to the best of our knowledge, a counterpart for
  France is not provided.
\item The preliminary investigation led us to use not cities but
  \emph{municipality} (Q15284), again including the \emph{subclass of} (P279)
  property, with more than 5000 inhabitants.
\item Since there are multiple french municipalities that are not located in the
  mainland of France, we encounter problems for constructing the metric
  space. To cope with that we draw a basic approximating square around the
  mainland of France and consider only those municipalities inside.
\item We find the class of every municipality, i.e, university location or
  non-university location, through the following approach.  We use the
  properties \emph{located in the administrative territorial entity} (P131) and
  \emph{headquarters location} (P159) on the set of all universities and checked
  if these are set in Germany or France. An example of a German University that
  has not set P131 is \emph{TU Dortmund} (Q685557).\footnote{last checked on
    19-06-25}
\item Using a Python script we then matched the list of municipalities with the
  indicated properties of the universities. This method was necessary for the
  following reason. Some universities are not related to municipalities through
  property P131. For example, the item \emph{Hochschule Niederrhein} (Q1318081)
  is located in the administrative location \emph{North Rhine-Westphalie}
  (Q1198), which is a federal state containing multiple municipalities. For these
  cases we checked the university locations manually. Some basic statistics on
  our dataset can be found in~\cref{table:basics}, a graphic overview of the
  municipality and university distribution is depicted in~\cref{fig:countries}.
\item During the construction of the data set we encounter universities that are
  associated to a country having neither \emph{located in the administrative
    territorial entity} (P131) nor \emph{headquarters location} (P159). There
  are ten German and twelve French universities for this case. We checked them
  manually and were able to discard them all for different reasons, for example,
  items that were  wrongly related to the university item.
\end{itemize}

\begin{table}
  \caption{Basic statistics of the country datasets extracted from
    wikidata.}
  \label{table:basics}
  \begin{center}
    \begin{tabular}{lrr}
      \toprule
      {} &  Municipalities &  University Locations \\
      \midrule
      France  &            2063 &                    92 \\
      Germany &            2863 &                   164 \\
      \bottomrule
    \end{tabular}
  \end{center}
\end{table}

\begin{figure*}[htbp]
    \includegraphics[trim=9em 0 0 0, width=0.8\columnwidth,
    clip]{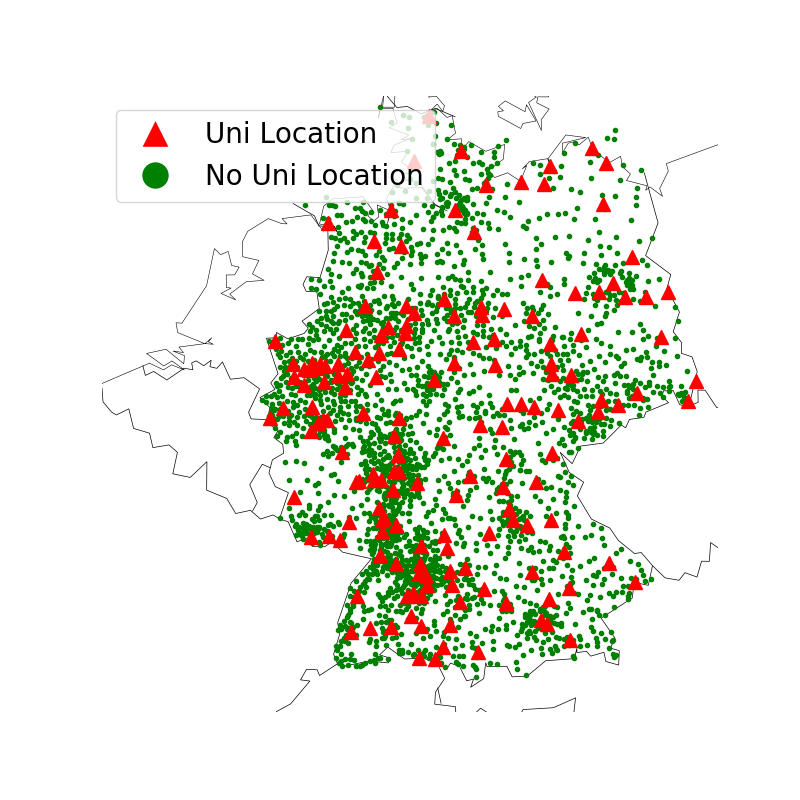}
    \includegraphics[trim=0 0 9em 0, width=0.8\columnwidth,
    clip]{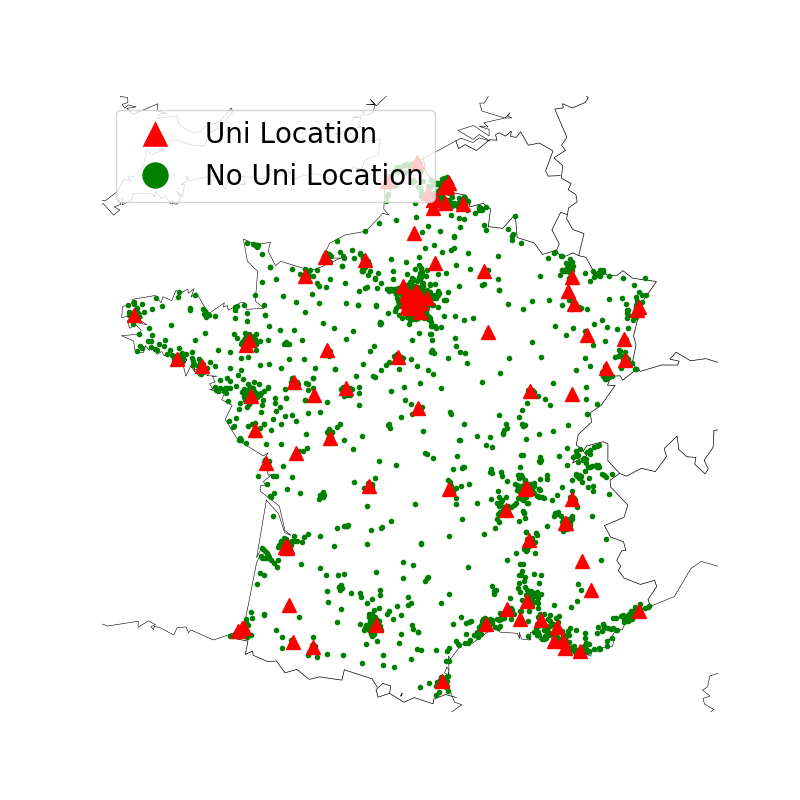}
  \caption{Municipalities in Germany (left) and France (right) having an
    university (Uni) and not having an university (No Uni).}
  \label{fig:countries}
\end{figure*}


\subsection{ Binary Classification Task }
\subsubsection*{Setup}
For both France and Germany, we compute the prominence and isolation of all
data points. We then normalize the population, isolation and prominence values
to be in the range from $0$ to $1$. Since our data set is highly imbalanced,
most of the common classifiers would tend to simply predict the majority
class. A variety of methods were proposed in the past to deal with such
problems. An overview can be found in~\cite{Kotsiantis06}. Sampling approaches
like undersampling or oversampling via the creation of synthetic
examples~\cite{Chawla02} are an established method for dealing with such
imbalances. We want to stress out again that the goal for the to be introduced
classification task is not to identify the best classifier. Rather we want to
produce evidence for the applicability of employing isolation and prominence as
(more suitable) features for learning a classification function. Since we need a
classification algorithm that provides useful predictions on single features, we
decide to use logistic regression with $L^2$ regularization and Support Vector
Machines~\cite{Cortes95} with a radial kernel. To overcome the imbalance, we use
inverse penalty weights with respect to the class distribution.

For our experiments, we use the algorithms for Support Vector Machines
(\texttt{SVC}) and \texttt{LogisticRegression} that are provided by the Python library
Scikit-Learn~\cite{Pedregosa12}. To solve the resulting minimization problem,
our setup of Scikit-Learn uses the LIBLINEAR library, see \cite{fan08}. As
penalty factor for the \texttt{SVC} we set $C=1$, however, we also experiment
with $C\in\{0.5,1,2,5,10,100\}$. As in~\cite{Akbani04}, where the authors
compared multiply methods to use support vector machines for imbalanced data
sets, we choose $\gamma=1$ for our radial kernel.  For all possible combinations of
population, isolation and prominence we use hundred iterations of five
cross-validation. We analyze to which extent the novel valuation functions help to
classify university municipalities in Germany and France.

\subsubsection*{Evaluation}

We use the g-mean (i.e., geometric mean) as evaluation function. Consider the confusion
matrix depicted in~\cref{tabdingohneschlaege}. 
\begin{table}[b]
  \caption{Confusion Matrix}
  \begin{tabular}{l|c|c}
    \toprule
    {} &  Predicted Negative &  Predicted Positive \\
    \midrule
    Actual Negative  &TN (True Negative) &FP (False Positive)\\
    Actual Positive &FN (False Negative) &TP (Tue Positive)\\
    \bottomrule
  \end{tabular}
    \label{tabdingohneschlaege}
\end{table}

Overall accuracy (i.e., how many test examples are classified correctly) is
highly misleading in the context of heavily imbalanced data. It is obvious that
for any classifier function predicting the majority would lead to an excellent
accuracy~\cite{Chawla10}. Therefore, we will evaluate the classification
decisions by using the geometric mean of the accuracy on the positive instances,
$acc_+ := \frac{TP}{TP+FN}$, often called sensitivity, and the accuracy on the
negative instances $acc_- :=\frac{TN}{TN+FP}$, often called specificity. Hence,
the g-mean score is then defined by the formula $g_{mean} :=\sqrt{acc_+ \cdot
  acc_-}$.  The evaluation function g-mean is established in the topic of
imbalanced data mining. It is mentioned in \cite{He09} and used
for evaluation in \cite{Akbani04}.

In our setup, the university locations are the positive class, meaning
that $acc_+$ corresponds to the classification results on the
university locations, and $acc_-$ corresponds to the accuracy on non
university locations.
For our experiments we now compare the values for g-mean for the following
cases. First, we train a classifier function purely on the features population,
prominence or isolation. Secondly, we also try combinations of them for the
training process. We consider in all those experiments the classifier solely
trained using the population feature as baseline, since this classification
function does not incorporate any metric aspects of the data set.
Then, an increase in g-mean when using prominence or isolation together with the
population function is evidence for the utility of the introduced valuation
functions. Furthermore, when directly comparing a classifier function that is
trained on isolation/prominence with a version trained on population, an
increase in g-mean strongly indicates the importance of the novel features.

In our experiments, we are not expecting high values for g-mean, since the
placement of university locations depends on many additional features, including
historical evolution of the country and political decisions. However, we claim
that the evaluation setup above is sufficient to show that the novel features
are potentially helpful for identifying interesting and useful items in
different tasks.

\subsubsection*{Results}

\begin{table*}[htbp]
  \begin{center}
    \caption{Results of the binary classification task. The results
      for every combination of features and every classifier. The best
      value for every combination of features is printed in bold.}
    po=population, pr=prominence, is=isolation \\
    SVM= Support Vector Machine, LR = Logistic Regression
    \label{tab:binary}
    \begin{tabular}{ll||rr|rr||rr|rr}
      \toprule
      Country &  & \multicolumn{4}{l}{France} & \multicolumn{4}{l}{Germany} \\
      Classifier &  & \multicolumn{2}{l}{SVM} & \multicolumn{2}{l}{LR} & \multicolumn{2}{l}{SVM} & \multicolumn{2}{l}{LR} \\
              & Score &   mean &    std &   mean &    std &    mean &    std &   mean &    std \\
      \midrule
      iso & acc+ & 0.5700 & 0.0075 & 0.6185 & 0.0036 &  0.5407 & 0.0044 & 0.6201 & 0.0043 \\
              & acc- & 0.9595 & 0.0010 & 0.9468 & 0.0010 &  0.9751 & 0.0003 & 0.9564 & 0.0009 \\
              & g-mean & 0.7395 & 0.0048 & 0.7652 & 0.0024 &  0.7261 & 0.0030 & 0.7701 & 0.0027 \\
      \midrule
      pr & acc+ & 0.2273 & 0.0041 & 0.3967 & 0.0065 &  0.1643 & 0.0035 & 0.3380 & 0.0075 \\
              & acc- & 1.0000 & 0.0000 & 0.9968 & 0.0004 &  1.0000 & 0.0000 & 0.9990 & 0.0002 \\
              & g-mean & 0.4767 & 0.0043 & 0.6288 & 0.0051 &  0.4054 & 0.0044 & 0.5811 & 0.0065 \\
      \midrule
      po & acc+ & 0.4684 & 0.0035 & 0.5815 & 0.0139 &  0.3370 & 0.0065 & 0.4949 & 0.0057 \\
              & acc- & 0.9932 & 0.0004 & 0.9834 & 0.0008 &  0.9970 & 0.0003 & 0.9886 & 0.0005 \\
              & g-mean & 0.6820 & 0.0025 & 0.7562 & 0.0092 &  0.5796 & 0.0056 & 0.6994 & 0.0041 \\
      \midrule
      iso+pr & acc+ & 0.5577 & 0.0100 & 0.6114 & 0.0088 &  0.5109 & 0.0075 & 0.5915 & 0.0061 \\
              & acc- & 0.9616 & 0.0011 & 0.9499 & 0.0008 &  0.9782 & 0.0006 & 0.9648 & 0.0009 \\
              & g-mean & 0.7323 & 0.0065 & 0.7621 & 0.0055 &  0.7069 & 0.0052 & 0.7554 & 0.0040 \\
      \midrule
      iso+po& acc+ & 0.6038 & 0.0131 & 0.6273 & 0.0050 &  0.6012 & 0.0055 & 0.6549 & 0.0061 \\
              & acc- & 0.9691 & 0.0011 & 0.9611 & 0.0007 &  0.9809 & 0.0007 & 0.9721 & 0.0005 \\
              & g-mean &\textbf{0.7649} & 0.0083 & \textbf{0.7764} & 0.0031 &  \textbf{0.7680} & 0.0035 & \textbf{0.7979} & 0.0037 \\
      \midrule
      pr+po & acc+ & 0.4770 & 0.0060 & 0.5543 & 0.0091 &  0.3524 & 0.0029 & 0.4966 & 0.0086 \\
              & acc- & 0.9960 & 0.0004 & 0.9895 & 0.0007 &  0.9978 & 0.0001 & 0.9927 & 0.0005 \\
              & g-mean & 0.6892 & 0.0044 & 0.7406 & 0.0061 &  0.5930 & 0.0024 & 0.7021 & 0.0061 \\
      \midrule
      iso+pr+po & acc+ & 0.5992 & 0.0115 & 0.6233 & 0.0066 &  0.5945 & 0.0053 & 0.6410 & 0.0075 \\
              & acc- & 0.9694 & 0.0011 & 0.9629 & 0.0008 &  0.9817 & 0.0006 & 0.9744 & 0.0006 \\
              & g-mean & 0.7622 & 0.0073 & 0.7747 & 0.0041 &  0.7640 & 0.0034 & 0.7903 & 0.0046 \\
      \bottomrule
    \end{tabular}
  \end{center}
\end{table*}

The results of our evaluation can be found in Table \ref{tab:binary}. In the
following we collect the observations drawn from this table.

\begin{inparaitem}
\item \textbf{Isolation is a good indicator for structural relevance.}
  Considering the results for both countries we notice that using isolation as
  the only feature leads to a solid prediction of university and non-university
  locations. For both countries and classifiers, it outperforms population.
\item \textbf{Combining absolute height with our valuation functions leads to
    better results.} Combining our orometric functions with population leads to
  better performance compared to solely the population feature.
\item \textbf{Prominence is not useful as a solo indicator.} Our result raises
  confidence that prominence alone is not an useful indicator for finding
  university locations. We may propose the following explanation. Prominence is
  a very strict valuation function: recall that we constructed the graphs by
  using distance margins as indicators for edges, leading to a dense graph
  structure in more dense parts of the metric space. It follows that a point in
  a more dense part has many neighbors and thus many potential paths that may
  lead to a very low prominence value. Observing
  definition~\cref{def:prom}, one can see that having a higher neighbor,
  with respect to the height function, always leads to a prominence value of
  zero. As mentioned earlier, the threshold is about 32 kilometers for Germany
  and 54 kilometers for France. Hence, a municipality has a not vanishing
  prominence if it is the most populated point in a radius of over 32
  kilometers, respectively 54 km. Only 75 municipalities of France have non zero
  prominence, with 41 of them being university locations. Germany has 124
  municipalities with positive prominence with 78 of them being university
  locations. Thus, prominence alone as a feature is insufficient for the
  prediction of university locations. As indicated in~\Cref{tab:binary}, the low
  g-mean score results from bad accuracy on the positive instances. Overall, it
  is an useful feature for identifying outstanding ``peaks''.
\item \textbf{The results for Germany differ from the results for France.} The
  margin in which isolation outperforms population as solely feature is for
  Germany greater than for France. The same holds for the score improvement if
  we add prominence and isolation as features to population. We assume that this
  observation is based on the difference in the geographic population
  distribution in France and in Germany: Having another look
  at~\Cref{fig:countries}, one may observe a tendency of clustering of
  university locations in some French areas. For example, looking at the area
  around Paris, one may observe a variety of universities located in the
  regional surrounding. The represented municipalities are all dominated by the
  nearby city Paris. As a consequence, they have a low isolation and prominence
  value.
\item \textbf{Support vector machine and logistic regression lead to similar
    results.} To the question, whether our valuation functions improve the
  classification compared with the population feature, support vector machines and
  logistic regressions provide the same answer: isolation always outperforms
  population, a combination of all features is always better then using just the
  plain population feature.
\item\textbf{Support vector machine penalty parameter.} Finally, for our last
  test we check the different results for support vector machines using the
  penalty parameters $C\in\{0.5,1,2,5,10,100\}$. We observe that increasing the
  penalty results in better performance using the population feature. However,
  for lower values of $C$, i.e., less overfitting models, we see better
  performance in using the isolation feature. In short, the more the model
  overfits due to $C$, the less useful are the novel valuation functions we
  introduced in this paper.
\end{inparaitem}

\section{Conclusion and Outlook}
\label{sec:conc}
In this work, we presented a novel approach to identify outstanding elements in
item sets. For this we employed orometric valuation functions, namely prominence
and isolation. We investigated a computationally reasonable transfer to the
realm of bounded metric spaces. In particular, we generalized previously known
results that were researched in the field of finite networks.

The theoretical work was motivated by the observation that knowledge graphs,
like Wikidata, do contain huge amounts of metric data. These are often equipped
with some kind of height functions in a natural way. Based on this we proposed
in this work the groundwork for an item recommending scheme. This envisioned
system would be capable of enriching conventional setups.

To evaluate the capabilities for identifying outstanding items we selected an
artificial classification task. We identified all French and German
municipalities from Wikidata and evaluated if a classifier can learn a
meaningful connection between our valuation functions and the relevance of a
municipality. To gain a binary classification task and to have a benchmark, we
assumed that universities are primarily located at relevant municipalities.  In
consequence, we evaluated if a classifier can use prominence and isolation as
features to predict university locations. Our results showed that isolation and
prominence are indeed helpful for identifying relevant items.

For future work we propose to develop the conceptualized item recommender system
and to investigate its practical usability in an empirical user
study. Furthermore, we urge to research the transferability of other orometric
based valuation functions. Finally, we acknowledge that our results about
valuation functions in metric spaces are surely already present in mathematical
theory. To identify the related mathematical notions and therefore to nourish
from advanced mathematical results would be the next theoretical goal.

\begin{acks}
  The authors would like to express thanks to Dominik Dürrschnabel for fruitful
  discussions.  This work was funded by the German Federal Ministry of Education
  and Research (BMBF) in its program ``Quantitative Wissenschaftsforschung'' as
  part of the REGIO project under grant 01PU17012.
\end{acks}

\citestyle{acmnumeric}
\bibliographystyle{ACM-Reference-Format}
\bibliography{literature}

\end{document}